\documentclass[10pt,journal,twocolumn]{IEEEtran}
\usepackage[type1]{tinos}
\usepackage{DejaVuSansMono}
\usepackage{amsmath,amssymb}
\usepackage{graphicx}
\usepackage{tikz}
\usetikzlibrary{arrows.meta}
\usepackage{booktabs,array}
\usepackage{tabularx}
\usepackage{enumitem}
\usepackage{amsthm}
\newtheorem{assumption}{Assumption}
\newtheorem{theorem}{Theorem}
\newtheorem{proposition}{Proposition}
\usepackage{url}
\usepackage{cite}
\usepackage{microtype}
\newcommand{\kw}[1]{\texttt{#1}}
\newcolumntype{L}{>{\raggedright\arraybackslash}X}
\usepackage[unicode]{hyperref}
\hypersetup{colorlinks=false,pdfborder={0 0 0}}
\graphicspath{{./}}

\title{\huge Runtime Control for AI Agents: Closing the\\ Evaluation--Execution Gap in Operational Workflows}
\author{\begingroup\setlength{\tabcolsep}{7pt}\begin{tabular}{cccc}
Jayesh Bhadane & Ramachandra Yerramsetti & Rajesh Gupta & Sumit Kumar\\
{\scriptsize\texttt{jayesh@runctrl.ai}} & {\scriptsize\texttt{rcyerramsetti@protonmail.com}} & {\scriptsize\texttt{rajesh@runctrl.ai}} & {\scriptsize\texttt{sumit.mamoria@gmail.com}}
\end{tabular}\endgroup}
\begin{document}
\maketitle

\begin{abstract}
Language-model agents increasingly participate in operational workflows where approvals, binds, and payments create commitments. An approval valid when issued may no longer authorize execution after relevant facts, evidence, policy, or authority change. We call this the evaluation--execution gap. We specify a state-bound runtime assurance layer that separates probabilistic measurement from deterministic adjudication, binds an execution allowance to the decision-relevant basis of an approval, invalidates that allowance after relevant logged changes, and rechecks admissibility at dispatch. This distinguishes the continued validity of a prior approval from a current-state policy check. We describe eight failure modes, an assurance model, a reference architecture, and prospective interventions with retrospective remediation. In a scripted 42-case replay, the state-bound configuration identifies 15 later-invalidated approvals that a deliberately non-re-evaluating point-in-time comparator does not revisit; it holds 14 affected actions before dispatch and compensates one landed effect. These results illustrate behavior under the stipulated policies and event orders, not superiority over dispatch-time enforcement designs or production effectiveness. Under explicit interception, ingestion, dependency, and serialization assumptions, the model establishes dispatch-time validity within the registered execution surface.
\end{abstract}

\section{Introduction}
\label{sec:intro}

Language-model agents plan, call tools, and act across long interactions. $\tau$-bench studies tool-agent-user work in real-world domains~\cite{yao2024tau}. Safety, robustness, and compliance remain open as actions need constraint, observation, attribution, and interruption.

In claims, underwriting, lending, payments, and benefits, an approval, bind, release, or decision creates an obligation. It must remain valid when it takes effect.

Consider a claims agent that may approve payouts below \$5{,}000 without senior review. It approves a claim at \$4{,}800. Before payment, a revised estimate raises the loss to \$7{,}200 and requires senior approval. The first approval was valid, but the case changed. If payment uses it, the decision is stale.

Agent evaluation assesses agent behavior~\cite{yehudai2025survey}. Agent observability traces execution artifacts and decisions~\cite{dong2024agentops}. In this paper, both supply measurement inputs. Runtime enforcement can constrain a proposed action~\cite{wang2026agentspec}. We add a controller that decides whether an earlier approval remains valid after facts, evidence, policy, or authority change.

We ask how an action can remain valid at execution after case, evidence, policy, or authority change. Our stateful decision runs before each action and after relevant change. It draws on runtime assurance and runtime verification~\cite{mehmood2022blackbox,leucker2009rv} and applies them to business commitments. The contributions are an eight-mode taxonomy (Section~\ref{sec:taxonomy}), a five-property decision (Section~\ref{sec:model}), an architecture with adjudication, invalidation, dispatch validity, and replay (Section~\ref{sec:arch}), seven interventions (Section~\ref{sec:interventions}), and a 42-case baseline comparison (Section~\ref{sec:evaluation}).

This specification paper positions the layer in Section~\ref{sec:related}, develops the examples in Section~\ref{sec:gap}, and states limits in Section~\ref{sec:discussion}.

\section{Background and Related Work}
\label{sec:related}

\textbf{Measurement.} Agent evaluation assesses multi-step behavior~\cite{yehudai2025survey}, while agent observability traces the artifacts and decisions behind that behavior~\cite{dong2024agentops}. At dispatch, the controller compares the approval record with the current case to decide whether the approval still applies.

\textbf{Runtime and process control.} Runtime verification, security automata, edit automata, and runtime assurance check or constrain execution~\cite{leucker2009rv,schneider2000eas,ligatti2005editautomata,mehmood2022blackbox}. Compliance monitoring and commitments reason over processes and obligations~\cite{ly2015compliance,singh1999commitments}. A conformant trace can still carry an approval the case has outgrown.

\textbf{Governance and access.} Delegation can supply logged identity, authority, and audit data~\cite{south2025delegation}. Attribute-based access control decides whether a principal may invoke an operation~\cite{hu2014abac}. UCON adds mutable conditions and ongoing checks~\cite{park2004ucon}. Yet live permission to invoke \kw{bind\_policy} can remain while the decision to bind this case becomes invalid.

\textbf{Transactions and rules.} Optimistic concurrency control validates a read set~\cite{kung1981occ}. Our business read set can span facts, approvals, policy, and days. Policy-as-code evaluates an arriving request, not an earlier approval. Table~\ref{tab:capabilities} shows the missing combination: a state-bound business decision that is re-evaluated and enforced at dispatch.

\begin{table*}[!t]
\caption{Relationship to prior mechanisms and the proposed state-bound decision-validity model.}
\label{tab:capabilities}
\centering
\footnotesize
\setlength{\tabcolsep}{3pt}
\renewcommand{\arraystretch}{1.07}
\begin{tabular}{@{}>{\raggedright\arraybackslash}p{.225\textwidth} >{\raggedright\arraybackslash}p{.355\textwidth} >{\raggedright\arraybackslash}p{.37\textwidth}@{}}
\toprule
Established approach & Existing capabilities & Design emphasis in this paper \\
\midrule
Agent evaluation and observability {[}2{]}, {[}3{]} & Behavioral evaluation, execution traces, and recorded assessments & Treat assessments as recorded inputs, not independently sufficient authorization for execution \\
Runtime enforcement and verification {[}4{]}--{[}8{]} & Runtime policy checks, temporal monitoring, and action constraints & Track the validity of a specific prior approval and invalidate its allowance after relevant state changes \\
Access and usage control {[}11{]}--{[}13{]} & Identity, delegation, contextual authorization, and mutable permissions & Separate endpoint permission from the standing associated with a specific recorded business approval \\
Process compliance and transaction control {[}9{]}, {[}10{]}, {[}14{]} & Process obligations, commitment models, and read-set validation & Bind approval history, evidence, decision dependencies, and rule versions to a long-lived execution allowance \\
\bottomrule
\end{tabular}
\end{table*}

\section{The Evaluation--Execution Gap}
\label{sec:gap}

Evaluation estimates risk. It does not decide whether an action should proceed under state, policy, authority, and reversibility.

\subsection{Why the gap is structural}

The gap is structural because approval, case change, and dispatch occur at different times. A point-in-time layer can judge the case at approval, but does not retain the facts, policy, evidence, and authority behind that decision. After a case changes, it cannot compare the old approval with the current case. At dispatch, an enforcement output cannot repair this missing history. A controller must retain the decision basis, re-evaluate it after relevant change, and, unlike probabilistic measurement, make a deterministic, reproducible, inspectable, and defensible intervention.

\emph{Runtime assurance} adjudicates each consequential action against current state before effect. \emph{Runtime control} is our deterministic layer. It maps upstream signals, state, policy, authority, evidence, and reversibility to one intervention.

\subsection{Two running examples}
\label{sec:examples}

\textbf{Claim payout.} The \$4,800 approval is valid under the specified \$5,000 rule. The later \$7,200 estimate invalidates its use for the payment. A dispatch-time threshold check would detect the amount increase; a state-bound allowance also records that the earlier approval was issued under different case conditions and must not be silently reused.

\textbf{Commercial-property submission 8871.} For this illustrative workflow, the declared underwriting rules require a senior sign-off at \$8.2M total insured value and independent retention of every materially relevant guideline exception. A parser records three separately identified guideline exceptions. The agent compresses them into ``one material issue'' without a recorded, authorized consolidation decision and recommends binding. A junior underwriter approves from that summary despite lacking the standing required by the exposure rule. In the sequence below, a senior underwriter then reviews the three separately recorded exceptions, resolves the unsupported summary, and issues a valid approval on that case state. Before binding, an inspection vendor reports a fourth exception. The example separates three questions: whether the agent's summary is supported, whether the reviewer had standing, and whether even a properly issued prior approval would remain valid after the new inspection. Table~\ref{tab:workflow} traces the submission across these positions, contrasting what the control layer retains with what a point-in-time view sees.

\begin{table*}[!t]
\caption{A state-bound decision across the 8871 underwriting workflow: what the control layer retains versus what a point-in-time view sees.}
\label{tab:workflow}
\centering
\scriptsize
\setlength{\tabcolsep}{3pt}
\renewcommand{\arraystretch}{1.07}
\begin{tabular}{@{}>{\raggedright\arraybackslash}p{.035\textwidth} >{\raggedright\arraybackslash}p{.19\textwidth} >{\raggedright\arraybackslash}p{.255\textwidth} >{\raggedright\arraybackslash}p{.205\textwidth} >{\raggedright\arraybackslash}p{.26\textwidth}@{}}
\toprule
Step & Workflow event & State retained by control layer & Point-in-time view & Runtime-control consequence \\
\midrule
t0 & Parser extracts three guideline exceptions & Three facts enter the record with source provenance & Extraction completes & Facts are available to future obligations \\
t1 & Agent writes ``one material issue; recommend bind'' & Narrative retained as an assertion, not a replacement for three facts & Summary can appear plausible in isolation & Recorded contradiction can trigger monitoring or escalation (F4) \\
t2 & Junior underwriter approves & Approval records actor, policy version, and current state & Review exists; user has tool access & Authority fails for \$8.2M exposure; route to senior authority (F1) \\
t3 & Senior reviews the three exceptions, resolves the summary, and approves & Valid approval supports an allowance bound to that case state & Approval is now present & Binding may proceed only while allowance dependencies remain unchanged \\
t4 & Inspection vendor reports a fourth exception & New fact changes variable read by binding obligations & Earlier senior approval still exists & Old allowance invalidated (F2) \\
t5 & Agent proposes \texttt{bind\_policy} & Current state includes four exceptions and allowance bound at t3 & Tool guard sees valid call plus approval & Re-adjudicate before dispatch; withhold binding until conditions are met \\
\bottomrule
\end{tabular}
\end{table*}

\section{A Taxonomy of Operational Workflow Failures}
\label{sec:taxonomy}

The examples above show stale decisions, missing standing, and unsupported assessments. The taxonomy names these and related implementation-independent failures in obligations, authority, evidence, and commitments. Later sections map each mode to detection and response.

\subsection{Workflows, commitments, and failure}

An \textbf{operational workflow} is a finite sequence of actor-performed steps with inputs, obligations, invariants, and successors. Obligations require standing and create \textbf{commitments} that may be reversible, compensable, or irreversible. Cases can change while open. An \textbf{operational failure} violates an invariant by action, commitment, or lack of progress, even without malfunction.

\subsection{Classification}

Each mode has a locus, surfacing horizon, and commitment reversibility. The locus can be evidence, authority, temporal validity, rules, record truthfulness, or boundary. It may surface during a step, later, or after an external trigger. Handoffs can delay detection because the next actor receives a summary. Reversibility sets response strictness.

\subsection{The failure modes}
\label{sec:modes}

\textbf{F1: Unauthorized commitment.} An approval lacks state-dependent business standing, despite tool permission. On 8871, the junior underwriter lacks senior standing. It may surface at audit and is irreversible.

\textbf{F2: Stale-decision execution.} A valid decision is invalidated by later case change before execution. Agents may execute against it rather than re-read the file. The claims payout is F2.

\textbf{F3: Incomplete-evidence discharge.} A step proceeds before policy-required facts are present. Agents need absence enforced as a block. A bind on 8871 before a signed statement of values is F3. It can surface at claim time and be irreversible.

\textbf{F4: Ungrounded discharge.} Facts and authority may exist, while the assessment does not follow from facts. The layer detects recorded contradiction, missing support, or a groundedness threshold. It cannot decide whether a supported assessment is substantively wrong (Section~\ref{sec:discussion}). On 8871, \emph{one material issue} contradicts three exceptions. The layer preserves the recorded facts and treats the assessment as a separate assertion.

\textbf{F5: Prohibited commitment.} No actor may make the commitment, so escalation cannot cure it. A bind in an unfiled class is F5 and may surface at regulatory examination.

\textbf{F6: Inconsistent-rule adjudication.} The governing policy version or its case applicability is not determinate. Effective dates, transitions, and case conditions matter. A transition assigning a case to neither version is F6.

\textbf{F7: Narrative capture.} A summary or recommendation replaces the facts it describes, or is stored as if it were a fact. If 8871 kept only \emph{one material issue} and dropped the three extracted exceptions, the controller could no longer compare the summary with the facts. When information enters the system, source typing keeps facts, assertions, and decisions separate.

\textbf{F8: Governance bypass.} A commitment uses an ungoverned database edit, payment rail, or override. No layer can intervene on what it never sees. F8 therefore lies outside the model's prevention guarantee.

\begin{table*}[t]\centering\footnotesize
\renewcommand{\arraystretch}{1.3}
\caption{Operational failure modes and their runtime images.}
\label{tab:modes}
\begin{tabularx}{\textwidth}{@{}L L L L L@{}}
\toprule
\textbf{Mode} & \textbf{Locus} & \textbf{Surfacing} & \textbf{Commitment} & \textbf{Intervention} \\
\midrule
F1 Unauthorized commitment & authority & audit & irreversible & \kw{escalate} \\
F2 Stale-decision execution & temporal validity & downstream / never & irreversible & \kw{rollback} \\
F3 Incomplete-evidence discharge & evidence & claim time & irreversible & \kw{require\_evidence} \\
F4 Ungrounded discharge & assessment & downstream / audit & varies & \kw{monitor} / \kw{escalate} \\
F5 Prohibited commitment & governing rules & regulatory exam & irreversible & \kw{deny} \\
F6 Inconsistent-rule adjudication & rule identity & audit & varies & \kw{pause} \\
F7 Narrative capture & veracity of record & pervasive & varies & structural (Sec.~\ref{sec:arch}) \\
F8 Governance bypass & process boundary & never & irreversible & excluded (Sec.~\ref{sec:arch}) \\
\bottomrule
\end{tabularx}
\end{table*}

\subsection{Two tiers, exhaustiveness, and scope}
\label{sec:exhaustive}

F1--F6 are \emph{adjudication failures}. The controller can see the proposed action and case record, then choose a response. F7--F8 are \emph{structural failures}. In F7, the record needed for a decision has been replaced or made unreliable. In F8, the action happens outside the controller's view.

The taxonomy is exhaustive only for the workflow step described in this paper. Within that boundary, a failure concerns missing input, an unsupported assessment, authority, an old decision, policy, record integrity, or a bypassed boundary. It does not cover infrastructure faults, attacks, or obligations that policy never encoded. An attack may still appear as F1--F7 if it changes a visible record, authority, or decision.

\section{The Runtime Assurance Model}
\label{sec:model}

F1--F6 are adjudicated when a governed action reaches the controller. A governed action is consequential, registered in the manifest, and routed through the controller. F2 can arise after a later event. The controller rebuilds the case, evaluates the action against current records, and returns one intervention. This section defines that protocol and its limited guarantee.

\subsection{The assurance decision}
\label{sec:decision}

The assurance decision, called D, answers one question: may this action proceed in the case as it stands now? D checks the proposed action against current case facts, applicable policy, the actor's authority, evidence, and reversibility. It returns one intervention. The controller records its reasons and returns the outcome to the agentic workflow, whose registered action connector enforces it.
\label{eq:decision}

Consider the claims payout introduced earlier. The agent approved \$4{,}800 under a rule that permits approval below \$5{,}000 without senior review. Before sending payment, the controller rebuilds the case from its initial state and ordered event log. The log is append-only: new events are added, and earlier events remain available. Reconstruction uses recorded events and preserved procedures, without a new model call or external lookup.
\label{eq:state-replay}

D needs more than the amount. It checks the approval, the actor's authority, required evidence, the governing policy version, and whether payment can be reversed or compensated. If all requirements hold, it returns \kw{allow}. The controller stores the action, case, actor, log position, outcome, and every input used. These inputs form the \emph{decision read set}. Each entry identifies the fact or governing artifact, its observed value, and its applicable version. This includes thresholds, authority and registry rules, prior decisions, deadlines, and effective dates. D uses only the action and this complete set of recorded inputs.

Now the revised estimate of \$7{,}200 arrives before payment. Ingestion adds it to the log. The amount is an input to the earlier decision, so the controller marks that decision stale and runs D again. The new amount requires senior review. The controller therefore returns \kw{escalate}; the agentic workflow holds payment and routes the claim for review. It records the senior review when it arrives and checks all current requirements again before returning an allowance.
\label{eq:freshness}

An unrelated event leaves the decision standing. A change to a recorded input, governing version, or version-applicability condition requires another decision. An expiry must enter the log before a later action that depends on it. Restoring an earlier value cannot revive the old decision, because the intervening change remains in the log. A digest identifies the decision for audit. The recorded inputs and subsequent events establish whether it remains current.

This separates an \emph{approval}, a recorded decision by a person or an authorized agent, from an \emph{allowance}, the controller's permission to proceed. The runtime returns its outcome and allowance to the agentic workflow. Immediately before its registered action connector invokes a target, it checks with the controller that the allowance remains current. A stale decision triggers re-evaluation and the connector holds the action. The final check and invocation must be serialized with incoming events, so no relevant change can enter between them.
\label{eq:dispatch}

D returns one of seven interventions. \kw{allow} permits execution. \kw{monitor} permits reversible or compensable work under observation. \kw{require\_evidence} names missing facts. \kw{escalate} routes to a reviewer. \kw{pause} waits for a condition to be resolved. \kw{deny} refuses the action. After an earlier reversible or compensable effect has landed, \kw{rollback} directs recovery. It does not authorize the original action. Table~\ref{tab:grid} maps deficiencies and reversibility to these outcomes.

\subsection{The inputs}
\label{sec:defs}

\textbf{Case state and action.} Case state contains the recorded facts, documents, decisions, and signals. In the payout example, it includes both estimates and the approval history. The action registry identifies the proposed operation and target as consequential.

\textbf{Policy and authority.} Versioned policy contains conditions, authority, evidence, and permitted actions. The controller selects a version only when its date, transition rule, case scope, and action scope fit. Otherwise the case is F6. An API right is not business standing. The authority check asks whether this actor may make this action for this case. It can depend on exposure, jurisdiction, product, delegation, and duty separation. Thus 8871's junior underwriter has access but not standing.

\textbf{Evidence.} Completeness checks required facts and detects F3. Provenance records source. Consistency compares assertions with recorded facts and detects F4 by contradiction or missing support. The layer cannot decide substantive correctness, so it routes that question.

\begin{table}[t]\centering\footnotesize
\renewcommand{\arraystretch}{1.25}
\caption{Evidence properties the layer can establish.}
\label{tab:evidence}
\begin{tabularx}{\columnwidth}{@{}l L@{}}
\toprule
\textbf{Property} & \textbf{Runtime layer} \\
\midrule
Completeness & establishes presence against declared requirements (F3) \\
Provenance & establishes source attribution from ingestion typing \\
Consistency & establishes agreement against recorded facts (F4) \\
Substantive correctness & cannot establish. Escalates \\
\bottomrule
\end{tabularx}
\end{table}

\textbf{Reversibility} is declared in the action registry as reversible, compensable, or irreversible. It detects no mode. It sets strictness, so doubt that permits monitoring for a reversible action can require escalation for an irreversible one.

\subsection{Properties}
\label{sec:properties}

\textbf{Determinism.} The same log prefix, action, decision-read-set values, and pinned governance-artifact versions give the same state and intervention. Signals enter only as logged values and thresholds.

\textbf{Separation of adjudication from extraction.} Parsers, evaluators, and people resolve unstructured material first. The decision makes no hidden model call.

\textbf{Risk-to-intervention translation.} Groundedness concern maps to \kw{monitor} for reversible work and \kw{escalate} for irreversible work.

\textbf{Reversibility-conditioned strictness.} Residual uncertainty needs stronger control before an irreversible commitment.

\textbf{Temporal re-evaluation.} A change to a recorded input, pinned version, or version-applicability condition triggers a new decision. This is why the \$7{,}200 revision stops payment under the earlier allowance.

\section{Reference Architecture}
\label{sec:arch}

The architecture comprises a manifest, interception boundary, typed log, reconstruction, engine, and controller. Figure~\ref{fig:architecture} shows the agentic layer receiving claim documents and running the agent workflow. The agent sends evidence references and a proposed action to runtime Ingest; the runtime records and adjudicates them, then returns an outcome to the agentic layer. In submission 8871, the inspection vendor's fourth exception arrives through the case workflow and is forwarded to Ingest. Because it changes a dependency of the pending binding decision, the runtime invalidates the old allowance and returns a hold or escalation. If an action is allowed, the registered action connector invokes it only after a current-allowance check. The runtime returns authorization to the agentic layer; it does not directly send the action to the enterprise target. Execution results return to the agentic layer and are forwarded to Ingest as logged events.

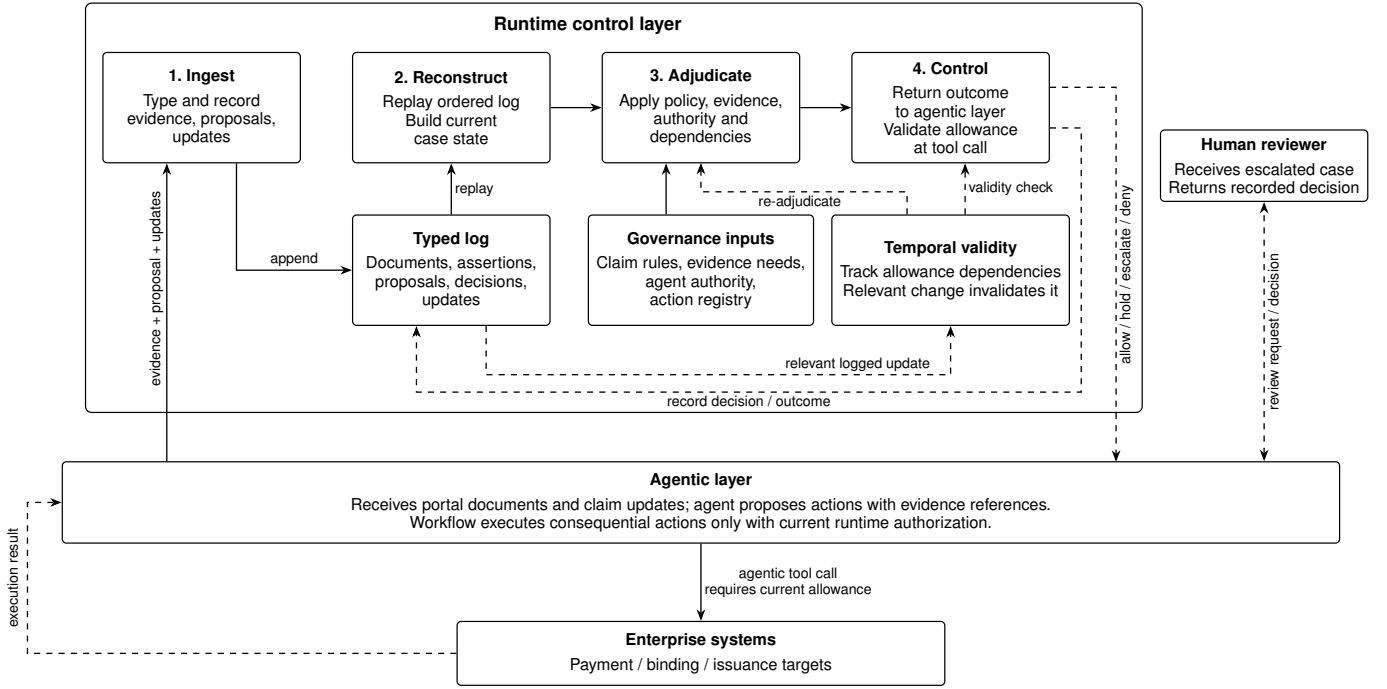
\begin{figure*}[t]
\centering
\resizebox{\textwidth}{!}{%
\begin{tikzpicture}[
  x=1cm,y=1cm,font=\sffamily\footnotesize,
  box/.style={draw=black,line width=.6pt,rounded corners=2pt,
    fill=none,align=center,inner sep=4pt},
  stage/.style={box,minimum width=3.4cm,minimum height=1.9cm},
  flow/.style={-{Stealth[length=2mm]},line width=.65pt,draw=black},
  feedback/.style={flow,dashed},
  reviewflow/.style={{Stealth[length=2mm]}-{Stealth[length=2mm]},line width=.65pt,draw=black,dashed},
  tag/.style={font=\sffamily\scriptsize,fill=none,inner sep=1.5pt,align=center}
]
\draw[rounded corners=3pt,draw=black,line width=.6pt,fill=none]
  (0,2.75) rectangle (18.2,9.8);
\node[anchor=north,font=\sffamily\bfseries\small] at (8.65,9.7)
  {Runtime control layer};

\node[stage] (ingest) at (2.0,8.0)
  {\textbf{1. Ingest}\\[3pt]Type and record\\evidence, proposals,\\updates};
\node[stage] (reconstruct) at (6.3,8.0)
  {\textbf{2. Reconstruct}\\[3pt]Replay ordered log\\Build current\\case state};
\node[stage] (adjudicate) at (10.6,8.0)
  {\textbf{3. Adjudicate}\\[3pt]Apply policy, evidence,\\authority and\\dependencies};
\node[stage] (control) at (14.9,8.0)
  {\textbf{4. Control}\\[3pt]Return outcome\\to agentic layer\\Validate allowance\\at tool call};

\node[stage] (log) at (6.3,5.2)
  {\textbf{Typed log}\\[3pt]Documents, assertions,\\proposals, decisions,\\updates};
\node[stage] (policy) at (10.6,5.2)
  {\textbf{Governance inputs}\\[3pt]Claim rules, evidence needs,\\agent authority,\\action registry};
\node[stage] (temporal) at (14.9,5.2)
  {\textbf{Temporal validity}\\[3pt]Track allowance dependencies\\Relevant change invalidates it};

\draw[flow] (reconstruct) -- (adjudicate);
\draw[flow] (adjudicate) -- (control);
\draw[flow] (log) -- node[tag,right]{replay} (reconstruct);
\coordinate (pin) at ([xshift=-.6cm]policy.north);
\draw[flow] (pin) -- (pin |- adjudicate.south);

\coordinate (appendOut) at ([xshift=.6cm]ingest.south);
\draw[flow] (appendOut) |- node[tag,above,pos=.75]{append} (log.west);

\coordinate (logOut) at ([xshift=.6cm]log.south);
\coordinate (updateLaneStart) at (logOut |- 0,3.4);
\coordinate (updateLaneEnd) at (14.9,3.4);
\draw[feedback] (logOut) |- (updateLaneStart)
  -- node[tag,above,pos=.8]{relevant logged update} (updateLaneEnd)
  -- (temporal.south);

\coordinate (validityCheck) at ([xshift=.25cm]temporal.north);
\coordinate (controlCheck) at ([xshift=.25cm]control.south);
\draw[feedback] (validityCheck) -- node[tag,right,pos=.52]{validity check} (controlCheck);

\coordinate (invalidateOut) at ([xshift=-.75cm]temporal.north);
\coordinate (recheckLane) at (10.6,6.55);
\coordinate (adjRecheck) at (10.6,7.05);
\draw[feedback] (invalidateOut) -- (invalidateOut |- 0,6.55)
  -- node[tag,below,pos=.52]{re-adjudicate} (recheckLane)
  -- (adjRecheck);

\coordinate (recOut) at ([yshift=-.35cm]control.east);
\coordinate (logIn) at ([xshift=-.6cm]log.south);
\draw[feedback] (recOut) -| (17.15,3.1) -- (logIn |- 0,3.1)
  node[tag,below,pos=.5]{record decision / outcome} -- (logIn);

\node[box,minimum width=22.0cm,minimum height=1.4cm] (agent) at (10.6,1.2)
  {\textbf{Agentic layer}\\[3pt]
   Receives portal documents and claim updates; agent proposes actions with evidence references.\\
   Workflow executes consequential actions only with current runtime authorization.};

\coordinate (ingIn) at ([xshift=-.6cm]ingest.south);
\draw[flow] (ingIn |- agent.north) -- (ingIn)
  node[tag,rotate=90,pos=.58,anchor=south]{evidence + proposal + updates};

\coordinate (outOut) at ([yshift=.35cm]control.east);
\draw[feedback] (outOut) -| (17.75,1.9)
  node[tag,rotate=90,pos=.75,anchor=north]{allow / hold / escalate / deny};

\node[box,minimum width=3.2cm,minimum height=1.2cm] (human) at (20.3,7.0)
  {\textbf{Human reviewer}\\[3pt]Receives escalated case\\Returns recorded decision};
\draw[reviewflow] (human.south |- agent.north) -- (human.south)
  node[tag,rotate=90,pos=.5,anchor=north]{review request / decision};

\node[box,minimum width=8.4cm,minimum height=1.1cm] (enterprise) at (10.6,-1.4)
  {\textbf{Enterprise systems}\\[3pt]Payment / binding / issuance targets};
\draw[flow] (agent.south) -- node[tag,right]{agentic tool call\\requires current allowance} (enterprise.north);
\draw[feedback] (enterprise.west) -| (-1.0,1.2) -- (agent.west);
\node[tag,rotate=90,anchor=south] at (-1.1,-.1) {execution result};

\end{tikzpicture}
}
\caption{Reference architecture for a runtime-controlled agentic workflow. Solid arrows show case inputs, proposals, processing, and agentic tool execution. Dashed arrows show runtime outcomes, review, decision recording, and invalidation. The runtime returns authorization to the agentic layer; its action connector invokes a target only with a current allowance.}
\label{fig:architecture}
\end{figure*}

\subsection{Interception}

The manifest registers each record mutation, its reversibility, and external binding. Running clearance is internal and reversible. Binding is external and irreversible. An engineer, not the agent, owns the manifest.

Every consequential tool invocation crosses the runtime authorization gate. The model agent cannot call a target directly; the registered action connector holds execution credentials and must obtain a current allowance from the runtime before invoking a consequential action. The agentic layer owns the workflow and initiates the invocation after authorization; the runtime returns allow, hold, escalation, or denial rather than sending the action itself. Alternate clients and APIs cannot bypass this topology. It treats F8.

Deployment enumerates mutations. Action-connector retries re-adjudicate with idempotency keys. Workers are governed or declared ungoverned. Privileged access is break-glass and logged. Consoles and legacy clients call the layer or fall outside the manifest. Third parties require notification. Non-layer writes are detectable. The guarantee covers the registered surface.

\subsection{The ledgers}
\label{sec:ingestion}

The append-only ordered log records results, decisions, outputs, signals, notifications, and policy changes. Typing records source, not truth. Facts are accountable observations, assertions are agent characterizations, and signals are evaluator judgements. Parser errors remain attributable. On 8871, three extracted exceptions are facts and \emph{one material issue, recommend bind} is an assertion. They coexist. Typing prevents F7 and separates extraction.

Human decisions record actor, claimed authority, policy version, and a rendered fingerprint of the decision record. The junior 8871 approval has a decision read set containing the three exceptions. A fourth exception changes that read set, making freshness and retrospective checks mechanical.

\subsection{State reconstruction}

State is rebuilt from the initial case state and ordered event log, without model or external reads. Correct replicas reconstruct the same state. Assertions do not change facts. Each adjudication records one log position for exact replay.

\subsection{The assurance engine}
\label{sec:engine}

The rules engine takes applicable obligations, thresholded signals, the business-authority result, and evidence absence. Policy, thresholds, and authority are immutable. On 8871, access passes but standing fails.

Each applicable obligation votes. The engine records all votes but returns the most restrictive one (Section~\ref{sec:severity}). A new obligation can raise, not lower, an outcome. Its record pins the action, decision read set, versions, votes, outcome, and log position for audit replay.

\subsection{The controller and temporal re-evaluation}
\label{sec:controller}

The controller records its outcome and returns it to the agentic layer. Escalation cites facts, contrary assertions, mismatched approvals, and obligations; the agentic workflow routes the case to a reviewer and returns the review decision as a logged event. The registered action connector checks the allowance again immediately before a consequential invocation. Irreversible actions are synchronous, so no approval is terminal.

A pending decision record is current only while none of its decision dependencies changes. Restoring a value cannot repair a stale record. The controller re-adjudicates affected actions and returns the updated outcome to the agentic layer. On 8871, a fourth exception stales the pending binding allowance before the action connector can invoke binding.

\subsection{Formal properties}
\label{sec:assumptions}

The following are architectural arguments for governed actions. They use the controller protocol in Section~\ref{sec:decision} and the assumptions below. They do not establish input truth, endpoint exactly-once effect, or behavior outside the governance boundary.

\begin{assumption}[Complete interception]
\label{as:intercept}
Every registered system-of-record mutation and externally binding action passes through the governance layer by credential topology, not convention.
\end{assumption}

\begin{assumption}[Authoritative event ingestion]
\label{as:state}
Every decision-relevant event, including expiry, enters the log with its affected dependency or version-applicability identity before a later action using it dispatches.
\end{assumption}

\begin{assumption}[Immutable governance artifacts]
\label{as:immutable}
The log is append-only and ordered. Its existing entries, initial case state, decision records, stored values, and policy, threshold, authority, registry, replay, and rule versions remain unchanged and available.
\end{assumption}

\begin{assumption}[Defined dispatch semantics]
\label{as:atomic}
For one action, the final validity check, outcome selection, and dispatch are serialized with additions to the log.
\end{assumption}

\begin{theorem}[Dispatch invariant]
\label{thm:validity}
Under Assumptions~\ref{as:intercept}--\ref{as:atomic}, every governed irreversible or externally binding action has a current decision record that permits its dispatch. The claim concerns governance-boundary dispatch. It does not establish execution, exactly-once effect, or atomic external logging.
\end{theorem}
\begin{IEEEproof}
The registered action connector sends each proposed consequential action to the controller and waits for its outcome. Ingestion records every prior relevant event. Under the rule in Section~\ref{sec:decision}, a relevant change makes the earlier record stale even if values later return. The connector invokes the target only after the controller returns permission and the final validity check confirms a current record. Serialization excludes a relevant event between that check and invocation. Thus each governed dispatch satisfies the invariant.
\end{IEEEproof}

\begin{theorem}[Deterministic adjudication]
\label{thm:determinism}
For a fixed log prefix, action, decision-read-set values, and pinned governance-artifact versions, the engine returns the same intervention.
\end{theorem}
\begin{IEEEproof}
Replay reconstructs the same state. Assumption~\ref{as:immutable} preserves the recorded inputs and procedures used by D. D reads those inputs without a new model or external call. Its fixed rules therefore return the same intervention.
\end{IEEEproof}

\begin{theorem}[Temporal invalidation]
\label{thm:invalidation}
A relevant event after a decision makes its record stale permanently, even if values return. Without such an event, the record remains current.
\end{theorem}
\begin{IEEEproof}
The validity rule examines all relevant events since the decision. Append-only history retains a change after reversal, so the old record cannot become current again. Without a relevant event, the rule leaves it current.
\end{IEEEproof}

\begin{proposition}[Audit replay]
\label{prop:replay}
Under Assumption~\ref{as:immutable}, any past adjudication can be reconstructed without the agent or model.
\end{proposition}
\begin{IEEEproof}
The retained record supplies the log position, action, read set, stored values, versions, votes, and outcome. Replay reconstructs the state, and Theorem~\ref{thm:determinism} reproduces the result.
\end{IEEEproof}

Authorization validity is not exactly-once execution. A lost acknowledgement can permit a valid retry to commit twice. The controller records it as pending with an idempotency key. Transactional endpoints append and effect atomically. Conditional writes reject old tokens. Weaker endpoints expose duplicates at reconciliation. Optimistic concurrency supplies compare-and-swap~\cite{kung1981occ}. The first two grades preserve the governed-dispatch guarantee. Input truth remains outside the guarantee.

\subsection{Component necessity}

Removing interception readmits F8. Removing typing readmits F7. Removing reconstruction or re-evaluation readmits F2. Removing authority, evidence, or policy versioning readmits F1/F5, F3, or F6. Admitting raw signals loses F4 detection or determinism. These components provide the stated coverage in this architecture.

\section{From Failures to Intervention}\label{from-failures-to-intervention}
\label{sec:interventions}

\subsection{Prospective intervention versus remediation}\label{prospective-intervention-versus-remediation}
\label{sec:severity}

Table~\ref{tab:grid} maps deficiency and reversibility to the typical prospective response. Six prospective outcomes address action proposals: allow, monitor, require\_evidence, escalate, pause, and deny. Rollback or compensation is a separate post-effect remediation operation. This removes the misleading implication that undoing an effect is merely a more restrictive form of denying a proposed action. In the accompanying implementation, rollback appears as a seventh entry in the same intervention list; the paper separates it because undoing a landed effect is not a stronger form of refusing a proposal.

Where multiple deficiencies coexist, the engine logs each one and applies the declared resolution policy---in the evaluated instance, a total restrictiveness order over the prospective outcomes. A landed effect may be reversed or compensated under a separate, authorized recovery policy. An irreversible effect cannot be rolled back; it can only be surfaced, reconciled, or remedied by another permissible process.

\begin{table}[!t]
\caption{Typical prospective response by deficiency and reversibility.}
\label{tab:grid}
\centering
\scriptsize
\setlength{\tabcolsep}{3pt}
\renewcommand{\arraystretch}{1.07}
\begin{tabular}{@{}>{\raggedright\arraybackslash}p{.27\columnwidth} >{\raggedright\arraybackslash}p{.315\columnwidth} >{\raggedright\arraybackslash}p{.31\columnwidth}@{}}
\toprule
Deficiency & Reversible or compensable & Irreversible \\
\midrule
None & allow & allow \\
Residual doubt (F4) & monitor or escalate & escalate \\
Missing evidence (F3) & require evidence & require evidence \\
Insufficient standing (F1) & escalate & escalate \\
Rule unclear (F6) & pause & pause \\
Action prohibited (F5) & deny & deny \\
Invalid prior allowance (F2) & \textbf{pause; re-adjudicate} & \textbf{pause; re-adjudicate} \\
\bottomrule
\end{tabular}
\end{table}

\section{Illustrative Evaluation}\label{illustrative-evaluation}
\label{sec:evaluation}

\subsection{Setup and scope}\label{setup-and-scope}

We evaluate the specification with a deterministic replay harness. A \emph{case} is a hand-written, ordered event log for one workflow, containing recorded facts, documents, approvals, evaluator signals, and agent proposals, together with the governing policy and threshold versions and a declared expected outcome. The harness replays each log through a controller configuration event by event and records every adjudication and dispatch decision. The library contains 42 such author-written cases: 27 include specified failures and 15 are designated clean, so that the prevented-violation rate is measured rather than trivially perfect. Scripted proposals and assertions replace a live agent; the harness makes no model, external-system, or wall-clock read. Its purpose is to check whether the specified adjudication, invalidation, and replay paths produce the intended outcomes under the stipulated policies and event sequences.

We compare (i) a point-in-time evaluation configuration that adjudicates a proposal once and carries its allowance forward without a later validity check, (ii) the proposed state-bound configuration using the same recorded policy, authority, and signal inputs, and (iii) a whole-state fingerprinting ablation that treats all case-state changes as relevant. The comparison isolates state-bound temporal invalidation relative to a deliberately non-re-evaluating configuration. It does not compare against a dispatch-time policy engine with independently maintained approval provenance, or establish advantage over all agent guardrails and runtime enforcement architectures.

\subsection{Results and interpretation}\label{results-and-interpretation}

Table~\ref{tab:results} reports the aggregate counts. All 15 later-invalidated approvals are identified by the state-bound configuration and not revisited by the point-in-time configuration. The state-bound configuration holds 14 affected actions before dispatch and addresses one compensable effect after it had landed. In the compensable case, a relevant state change occurred after valid dispatch but before the external effect completed; the resulting effect required compensation rather than prospective prevention. These outcomes should not be collapsed into a single prevention statistic.

Eight failures already present at proposal time are caught by both configurations; two of these are attempts to circumvent in-path controls--- an agent re-requesting dispatch of a denied action, and an agent asserting over a disqualifying fact---refused by both.

The 15 temporal invalidations and the five cases primarily labeled F2 describe different categorizations: temporal invalidation can expose an authority, evidence, grounding, prohibition, or rule problem; Table~\ref{tab:permode} lists each affected case's primary deficiency.

Two rates summarize the trade-off, with denominators fixed by the harness. The adverse-effect rate is the share of the 23 governed actions with a non-dispatch outcome under the final stipulated case state that nonetheless produce an external effect: 65.2\% (15/23) point-in-time versus 4.3\% (1/23) state-bound. The one state-bound effect was dispatched while valid and became adverse after a later state change but before external completion.

The escalation rate is the share of all 43 proposed actions routed to a human: 11.6\% (5/43) versus 32.6\% (14/43). The reduction in adverse external effects is accompanied by additional human attention; both rates describe the scripted replay, not production incidence. The 27 failure cases contain 28 governed actions; five are legitimately dispatched by both configurations because their deficiency is remedied in the event log before dispatch, leaving the 23 whose adjudicated outcome is non-dispatch.

One of the 15 clean cases requires an additional hold even though fresh adjudication permits the action: its amount changes within the same authority limit, invalidating the old allowance without rendering the new amount inadmissible. This is an unnecessary repeat decision, not an assertion that the final state violates policy. It motivates calibration of safe dependency scope while preserving soundness.

Table~\ref{tab:ablation} reports the fingerprinting ablation. The five-case stress set places 10--20 irrelevant events after pending allowances and contains no stale approvals. Whole-state fingerprinting produces 74 re-adjudications and five unnecessary invalidations; the declared decision-relevant projection produces none. The two configurations detect the same 15 stale approvals in the 42-case library. These results illustrate the selectivity mechanism for the dependencies encoded by the authors, not the reliability of automatically discovering all real-world dependencies.

Each case was replayed three times under each configuration, with byte-identical decision records and votes. This checks reproducibility of the harness's deterministic path; it is not a latency, scalability, or live operational reliability result.

\begin{table}[!t]
\caption{Reported counts from the 42-case replay. Denominators are stated where they differ from the number of cases; adverse-effect interpretation requires verification against the replay sequence.}
\label{tab:results}
\centering
\scriptsize
\setlength{\tabcolsep}{3pt}
\renewcommand{\arraystretch}{1.07}
\begin{tabular}{@{}>{\raggedright\arraybackslash}p{.62\columnwidth} >{\centering\arraybackslash}p{.155\columnwidth} >{\centering\arraybackslash}p{.155\columnwidth}@{}}
\toprule
Measure & Point-in-time & State-bound \\
\midrule
Later-invalidated approvals identified (15) & 0 & 15 \\
Affected actions held before dispatch & 0 & 14 \\
Affected landed effects compensated & 0 & 1 \\
Proposal-time failures caught by both & 8 & 8 \\
Clean cases unnecessarily held (15) & 0 & 1 \\
\textbf{Adverse external effects} (of 23 governed actions with non-dispatch final outcome) & 15 & 1 \\
Escalations to a human (of 43 proposed actions) & 5 & 14 \\
Byte-identical adjudication on repetition & yes & yes \\
\bottomrule
\end{tabular}
\end{table}
\begin{table}[!t]
\caption{Primary deficiency assigned to the 15 affected cases. An approval can additionally be temporally invalidated; F2 is not exclusive.}
\label{tab:permode}
\centering
\scriptsize
\setlength{\tabcolsep}{3pt}
\renewcommand{\arraystretch}{1.07}
\begin{tabular}{@{}>{\raggedright\arraybackslash}p{.81\columnwidth} >{\centering\arraybackslash}p{.11\columnwidth}@{}}
\toprule
Primary case label & Cases \\
\midrule
F1 Authority/standing & 3 \\
F2 Staleness without a more specific primary label & 5 \\
F3 Evidence & 2 \\
F4 Recorded grounding & 1 \\
F5 Prohibition & 1 \\
F6 Rule identity & 3 \\
\textbf{Total affected cases} & \textbf{15} \\
\bottomrule
\end{tabular}
\end{table}
\begin{table}[!t]
\caption{Whole-state versus decision-relevant fingerprinting ablation. ``False inv.'' denotes invalidation caused by events not requiring a changed substantive authorization; a relevant change can still require an unnecessary repeat approval when readjudication permits the action.}
\label{tab:ablation}
\centering
\scriptsize
\setlength{\tabcolsep}{3pt}
\renewcommand{\arraystretch}{1.07}
\begin{tabular}{@{}>{\raggedright\arraybackslash}p{.50\columnwidth} >{\centering\arraybackslash}p{.115\columnwidth} >{\centering\arraybackslash}p{.115\columnwidth} >{\centering\arraybackslash}p{.135\columnwidth}@{}}
\toprule
Configuration & Re-adj. & False inv. & Stale det. \\
\midrule
Whole-state, library & 23 & 1 & 15/15 \\
Decision-relevant, library & 19 & 0 & 15/15 \\
Whole-state, stress set & 74 & 5 & n/a \\
Decision-relevant, stress set & 0 & 0 & n/a \\
\bottomrule
\end{tabular}
\end{table}

\subsection{What the replay cannot establish}\label{what-the-replay-cannot-establish}

The authors defined the rules and cases, and the point-in-time comparator is intentionally unable to re-evaluate. Hence the counts demonstrate specification behavior under constructed conditions. A dispatch-time policy re-evaluator could detect many of the same changes. The additional semantics proposed here are the bound history of which prior approval is being used, what evidence and authority supported it, and whether that approval has been superseded. No head-to-head experiment against such a stronger comparator is claimed.

The scripted counts are not estimates of production incident prevalence or review capacity. In particular, the system may increase the number of repeat approvals or escalations; the practical burden requires measurement in a live deployment. The current paper does not report production latency or claim an empirical advantage over alternative implemented architectures.

\section{Discussion and Limitations}\label{discussion-and-limitations}
\label{sec:discussion}

\textbf{Interception.} An unfronted payment rail or database edit is F8. The controller cannot prevent an effect outside its registered path. The dispatch-time property is conditional on credential and integration topology.

\textbf{Ingestion and external effects.} A relevant update that has not yet reached the log can invalidate the real-world business basis without violating the stated theorem. State may also change after a valid dispatch but before an external provider commits. The architecture guarantees neither provider-side atomicity nor exactly-once external effect; compensation and reconciliation remain necessary.

\textbf{Dependency completeness.} Selective invalidation is sound only if the declared read set includes derived, cross-step, and policy-applicability dependencies. A missing dependency permits a falsely retained approval. Unknown dependencies call for conservative widening or review, at the cost of additional intervention.

\textbf{Extraction and evidence fidelity.} A parser can misrecord a loss amount and yield a reproducibly wrong decision. Typing preserves the origin of assertions but does not prove factual truth. Completeness checks cannot establish substantive professional correctness.

\textbf{Policy and workflow scope.} The controller enforces encoded obligations; an unencoded rule is unenforceable. It does not solve general planning failure, deadlock, business desirability, or professional judgment that cannot be represented as an obligation. A human reviewer may still make a poor decision even with accurate recorded context.

\textbf{Next steps.} Production incident libraries, stronger dispatch-time comparators, integration measurements, and review-load studies would establish how well this specification transfers outside the hand-written cases. These are future work, not unperformed experiments implied by the current results.

\section{Conclusion}\label{conclusion}
\label{sec:conclusion}

A business decision can be valid when issued and unusable when an agent later acts. This paper specifies an explicit state-bound allowance that records the decision basis, distinguishes human approval from current execution permission, invalidates associated allowances after relevant logged changes, and requires a new admissibility check at the registered dispatch boundary. It draws on established runtime enforcement, mutable authorization, process compliance, and transaction-control mechanisms, with a particular focus on long-lived business decisions. The 42-case scripted replay illustrates its temporal invalidation and selective fingerprinting behavior. Within the specified assumptions and registered surface, the resulting guarantee concerns validity at dispatch, not universal business correctness or the timing and uniqueness of external effects.

\bibliographystyle{IEEEtran}
\bibliography{references}

@article{yao2024tau,
  title         = {$\tau$-bench: A Benchmark for Tool-Agent-User Interaction in Real-World Domains},
  author        = {Yao, Shunyu and Shinn, Noah and Razavi, Pedram and Narasimhan, Karthik},
  journal       = {arXiv preprint arXiv:2406.12045},
  year          = {2024},
  eprint        = {2406.12045},
  archivePrefix = {arXiv},
  primaryClass  = {cs.AI}
}

@article{wang2026agentspec,
  title         = {{AgentSpec}: Customizable Runtime Enforcement for Safe and Reliable {LLM} Agents},
  author        = {Wang, Haoyu and Poskitt, Christopher M. and Sun, Jun},
  journal       = {arXiv preprint arXiv:2503.18666},
  year          = {2025},
  eprint        = {2503.18666},
  archivePrefix = {arXiv},
  primaryClass  = {cs.SE}
}

@inproceedings{mehmood2022blackbox,
  title         = {The Black-Box Simplex Architecture for Runtime Assurance of Autonomous {CPS}},
  author        = {Mehmood, Usama and Sheikhi, Sanaz and Bak, Stanley and Smolka, Scott A. and Stoller, Scott D.},
  booktitle     = {NASA Formal Methods (NFM)},
  series        = {Lecture Notes in Computer Science},
  volume        = {13260},
  pages         = {231--250},
  year          = {2022},
  publisher     = {Springer},
  doi           = {10.1007/978-3-031-06773-0_12},
  eprint        = {2102.12981},
  archivePrefix = {arXiv},
  primaryClass  = {eess.SY}
}

@article{leucker2009rv,
  title   = {A Brief Account of Runtime Verification},
  author  = {Leucker, Martin and Schallhart, Christian},
  journal = {Journal of Logic and Algebraic Programming},
  volume  = {78},
  number  = {5},
  pages   = {293--303},
  year    = {2009},
  doi     = {10.1016/j.jlap.2008.08.004}
}

@article{schneider2000eas,
  title   = {Enforceable Security Policies},
  author  = {Schneider, Fred B.},
  journal = {ACM Transactions on Information and System Security},
  volume  = {3},
  number  = {1},
  pages   = {30--50},
  year    = {2000},
  doi     = {10.1145/353323.353382}
}

@article{ligatti2005editautomata,
  title   = {Edit Automata: Enforcement Mechanisms for Run-time Security Policies},
  author  = {Ligatti, Jay and Bauer, Lujo and Walker, David},
  journal = {International Journal of Information Security},
  volume  = {4},
  number  = {1--2},
  pages   = {2--16},
  year    = {2005},
  doi     = {10.1007/s10207-004-0046-8}
}

@article{ly2015compliance,
  title   = {Compliance Monitoring in Business Processes: Functionalities, Application, and Tool-Support},
  author  = {Ly, Linh Thao and Maggi, Fabrizio Maria and Montali, Marco and Rinderle-Ma, Stefanie and van der Aalst, Wil M. P.},
  journal = {Information Systems},
  volume  = {54},
  pages   = {209--234},
  year    = {2015},
  doi     = {10.1016/j.is.2015.02.007}
}

@article{singh1999commitments,
  title   = {An Ontology for Commitments in Multiagent Systems: Toward a Unification of Normative Concepts},
  author  = {Singh, Munindar P.},
  journal = {Artificial Intelligence and Law},
  volume  = {7},
  number  = {1},
  pages   = {97--113},
  year    = {1999},
  doi     = {10.1023/A:1008319631231}
}

@article{south2025delegation,
  title         = {Authenticated Delegation and Authorized {AI} Agents},
  author        = {South, Tobin and Marro, Samuele and Hardjono, Thomas and Mahari, Robert and Whitney, Cedric Deslandes and Greenwood, Dazza and Chan, Alan and Pentland, Alex},
  journal       = {arXiv preprint arXiv:2501.09674},
  year          = {2025},
  eprint        = {2501.09674},
  archivePrefix = {arXiv},
  primaryClass  = {cs.AI}
}

@techreport{hu2014abac,
  title       = {Guide to Attribute Based Access Control ({ABAC}) Definition and Considerations},
  author      = {Hu, Vincent C. and Ferraiolo, David and Kuhn, Rick and Schnitzer, Adam and Sandlin, Kenneth and Miller, Robert and Scarfone, Karen},
  institution = {National Institute of Standards and Technology},
  number      = {NIST Special Publication 800-162},
  year        = {2014},
  doi         = {10.6028/NIST.SP.800-162}
}

@article{park2004ucon,
  title   = {The {UCON}$_{ABC}$ Usage Control Model},
  author  = {Park, Jaehong and Sandhu, Ravi},
  journal = {ACM Transactions on Information and System Security},
  volume  = {7},
  number  = {1},
  pages   = {128--174},
  year    = {2004},
  doi     = {10.1145/984334.984339}
}

@article{kung1981occ,
  title   = {On Optimistic Methods for Concurrency Control},
  author  = {Kung, H. T. and Robinson, John T.},
  journal = {ACM Transactions on Database Systems},
  volume  = {6},
  number  = {2},
  pages   = {213--226},
  year    = {1981},
  doi     = {10.1145/319566.319567}
}

@article{yehudai2025survey,
  title         = {Survey on Evaluation of {LLM}-based Agents},
  author        = {Yehudai, Asaf and Eden, Lilach and Li, Alan and Uziel, Guy and Zhao, Yilun and Bar-Haim, Roy and Cohan, Arman and Shmueli-Scheuer, Michal},
  journal       = {arXiv preprint arXiv:2503.16416},
  year          = {2025},
  eprint        = {2503.16416},
  archivePrefix = {arXiv},
  primaryClass  = {cs.AI}
}

@article{dong2024agentops,
  title         = {{AgentOps}: Enabling Observability of {LLM} Agents},
  author        = {Dong, Liming and Lu, Qinghua and Zhu, Liming},
  journal       = {arXiv preprint arXiv:2411.05285},
  year          = {2024},
  eprint        = {2411.05285},
  archivePrefix = {arXiv},
  primaryClass  = {cs.AI}
}

\end{document}